\documentclass[wcp]{jmlr}

\usepackage{longtable}% for long tables

\usepackage{booktabs}
\usepackage{algorithm}
\usepackage{algorithmic}
\usepackage{graphicx}
\DeclareSymbolFont{operators}{OT1}{cmr}{m}{n}
\DeclareMathAlphabet{\mathcal}{OMS}{cmsy}{m}{n}
\DeclareSymbolFont{letters}{OML}{cmm}{m}{it}
\DeclareSymbolFont{symbols}{OMS}{cmsy}{m}{n}
\DeclareSymbolFont{largesymbols}{OMX}{cmex}{m}{n}
\SetSymbolFont{operators}{bold}{OT1}{cmr}{bx}{n}

\newcommand{\ltheta}{\nu_{\bar{\theta}}}
\newcommand{\lomega}{\nu_{\bar{\omega}}}

\newcommand{\adaBracket}[1]{\left( #1 \right)}
\newcommand{\adaRectBracket}[1]{\left[ #1 \right]}

\newcommand{\rewardNoT}{\mathcal{R}(s, a)}

\newcommand{\EqNumber}[1]{Eq. (\ref{#1})}
\newcommand{\myNorm}[1]{\left|\!\left| #1  \right|\!\right|_{\infty}}
\newcommand{\sumJK}[2]{\sum_{j=#1}^{#2}}
\newcommand{\sumJ}[1]{\sum_{j=#1}^{k}}

\def\Eqref#1{Eq.~(\ref{#1})}
\def\Figref#1{Figure~\ref{#1}}

\def\Lemmaref#1{Lemma~\ref{#1}}

\DeclareMathOperator*{\argmax}{arg\,max}

\newcommand{\state}{\mathcal{S}}
\newcommand{\action}{\mathcal{A}}
\newcommand{\reward}{\mathcal{R}}

\newcommand{\trans}{{P}\left(s^{\prime} \mid s, a\right)}

\newcommand{\approxadv}{\tilde{\mathbb{{A}}}}

\newcommand{\entreg}[2]{\mathcal{I}_{#1}^{#2}}
\newcommand{\cvival}[2]{V_{#1}^{#2}}
\newcommand{\cviaval}[2]{Q_{#1}^{#2}}

\newcommand{\vparams}{\theta}
\newcommand{\pparams}{\phi}

\newcommand{\partition}{Z(\sv)}
\newcommand{\rbuf}{\mathcal{B}}
\newcommand{\onrbuf}{\mathcal{B}_{K}}

\newcommand{\lphi}{\nu_{\bar{\phi}}}
\newcommand{\lA}{\nu_{A}}
\newcommand{\lAslow}{\nu_{A^{\mathrm{MaxDiff}}}}

\newcommand{\clip}{\mathop{\rm clip}}

\newcommand{\return}{J}

\newcommand{\basepol}{\tilde\pi}

\newcommand{\pol}{\pi}

\newcommand{\avg}[1]{\mathbb{E}^{#1}}
\newcommand{\reals}{\mathbb{R}}
\newcommand{\gr}{\mathcal{G}}
\newcommand{\const}{\mathcal{X}}
\newcommand{\av}{a}
\newcommand{\sv}{s}
\newcommand{\nv}{\mathbf{\epsilon}}

\newcommand{\svp}{s^\prime}

\newcommand{\st}{s_t}

\newcommand{\stp}{s_{t+1}}

\usepackage[symbol]{footmisc}

\newcommand*\samethanks[1][\value{footnote}]{\footnotemark[#1]}
\renewcommand{\thefootnote}{\fnsymbol{footnote}}

\makeatletter
\let\Ginclude@graphics\@org@Ginclude@graphics 
\makeatother

\title[Cautious Actor-Critic]{Cautious Actor-Critic}

  \author{\Name{Lingwei Zhu}\thanks{Equal contribution}  \Email{lingwei.andrew.zhu@gmail.com} \\
  \Name{Toshinori Kitamura}\samethanks \Email{kitamura.toshinori.kt6@is.naist.jp}\\
  \Name{Takamitsu Matsubara} \Email{takam-m@is.naist.jp}\\
  \addr Nara Institute of Science and Technology, JAPAN
 }

\begin{document}

\maketitle

\begin{abstract}
    The oscillating performance of off-policy learning and persisting errors in the actor-critic (AC) setting call for algorithms that can conservatively learn to suit the stability-critical applications better.
    In this paper, we propose a novel off-policy AC algorithm cautious actor-critic (CAC).
    The name cautious comes from the doubly conservative nature that we exploit the classic policy interpolation from conservative policy iteration for the actor and the entropy-regularization of conservative value iteration for the critic.
    Our key observation is the entropy-regularized critic facilitates and simplifies the unwieldy interpolated actor update while still ensuring robust policy improvement.
    We compare CAC to state-of-the-art AC methods on a set of challenging continuous control problems and demonstrate that CAC achieves comparable performance while significantly stabilizes learning. 

\end{abstract}


\section{Introduction}\label{sec:intro}

Actor-critic (AC) methods of reinforcement learning (RL) have been gaining increasing interests recently due to their scalability  to large-scale problems: they can learn with both on-policy or/and off-policy samples and handle continuous action spaces \citep{lillicrap2015continuous,trpo-schulman15}; both in model-free or model-based setting \citep{pmlr-v80-haarnoja18b,hafner2020-dream}. 
%For model-free learning, drawing inspiration from the rich literature of approximate dynamic programming (ADP) \citep{munos2016safe} has helped in establishing remarkable performance \citep{wang2017sample,fakoor2020p3o}, especially the state-of-the-art soft actor-critic (SAC) \citep{pmlr-v80-haarnoja18b} that connects with the fruitful studies of Shannon-entropy-regularized ADP \citep{ziebart2010-phd,haarnoja-2017a}.
Recently in the model-free setting there has seen a booming in off-policy AC methods \citep{wang2017sample,pmlr-v80-haarnoja18b,fakoor2020p3o}. However, while these methods are sample-efficient in exploiting off-policy samples for continuous control, it is those samples that often bring oscillating performance during learning as a side-effect due to distribution mismatch. The oscillating performance of off-policy learning and persisting errors in the AC setting \citep{pmlr-v80-fujimoto18a} call for algorithms that can conservatively learn to better suit the stability-critical applications.

The performance oscillation and degradation problems have been widely discussed in the approximate dynamic programming (ADP) literature \citep{wagner2011,Bertsekas2011} that has motivated efficient learning algorithms against various sources of error. 
%The abundance of relevant literature naturally motivates one to look at existing ADP methods for possible solutions applicable to AC.
%ADP methods, though limited to discrete action spaces, often provide valuable insights for designing successful AC methods.
%e.g. the Retrace$(\lambda)$ algorithm \citep{munos2016safe} for ACER \citep{wang2017sample} and soft Q-learning \citep{haarnoja-2017a} for SAC \citep{pmlr-v80-haarnoja18b}. 
%In the ADP literature, 
The seminal work of \citep{kakade-cpi} propose a principled approach to tackle performance degradation by leveraging policy interpolation which is \emph{conservative} in that it reduces greediness of the updated policy.
However, though it enjoys strong theoretical guarantees, its drawbacks limit its use in the AC setting: (1) it is difficult to obtain a reliable reference policy in high-dimensional continuous state-action spaces; (2) the interpolation is often regarded inconvenient to use and it is unclear how to design the interpolation in continuous action spaces.
In practice, two popular variants \citep{trpo-schulman15,schulman1707proximal} that sidestep the interpolation and directly approximate the updated policy are more often used in the AC setting.
%However, though this method very recent research extends this idea to learning with deep networks \citep{Vieillard-2020DCPI}, further extending it to continuous action space settings remains unresolved due to the difficulty in designing the interpolation coefficient \citep{kakade-cpi}.
On the other hand, the recently booming entropy-regularized ADP literature \citep{azar2012dynamic,Fox2016,kozunoCVI,vieillard2020leverage} also features \emph{conservative learning} \citep{kozunoCVI} as they average over past value functions \citep{vieillard2020leverage}. 
Though these methods do not explicitly address the performance oscillation problem, they have been empirically verified to be error-tolerant and yield state-of-the-art performance on a wide range of tasks. 
Extending this conservative learning to AC has been studied \citep{nachum2017trust,pmlr-v80-haarnoja18b}. However, the resulted \emph{conservativeness} exists only in the critic: In challenging tasks, the performance degradation and oscillation still occur.

%On the other hand, there are research areas in ADP where advances have been made but has not seen the counterparts in AC. For entropy-regularization, it has been studied in detail by \citep{kozunoCVI,vieillard2020leverage} when both Shannon entropy and Kullback-Leibler (KL) divergence present as regularization. For the classic conservative policy iteration (CPI) \citep{kakade-cpi}, \citet{pmlr-v28-pirotta13,abbasi-improvement16} extend the theory and \citet{Vieillard-2020DCPI} successfully apply it to learning with deep networks.
%However, we are unaware of published results that carry those ideas to AC.
%This might be due to the discrete summation over the action space in ADP becomes intractable $\argmax$ or integrals over the continuous action space.  \citet{lillicrap2015continuous,pmlr-v80-haarnoja18b} adopt different approaches for solving this problem as the policy is obtained by deterministic $\argmax$ and the non-deterministic softmax. Identifying suitable methods for evaluating the policy in continuous action spaces is potentially challenging.
%To the best of our knowledge, we are unaware of any published results that carry these ideas to AC.

This paper aims to tackle the performance oscillation problem of off-policy AC by proposing a novel algorithm: \emph{cautious actor critic} (CAC), where the naming cautious comes from the \emph{doubly conservative} nature as we combine a conservative actor leveraging the concept of conservative policy iteration (CPI) \citep{kakade-cpi} with a conservative critic exploiting the entropy-regularization of conservative value iteration (CVI) \citep{kozunoCVI}.
%\sout{The key observation is that when the critic is entropy-regularized, training the actor exploiting interpolated policies can be significantly simplified, which in turn ensures robust policy improvement. Hence, CAC takes the best of both lines of ADP research on conservative learning into the AC setting. }
The key observation is that the entropy-regularized critic can find error-tolerant reference policies and simplifies the unwieldy interpolated actor update while still ensures robust policy improvement. 
CAC leverages automatically adjusted interpolation to reflect the faith during learning: when performance oscillation is likely to happen, CAC behaves cautiously to rely more on validated previous policy rather than on the new policy.
Our novel interpolation design is inspired by a very recent study from the ADP literature \citep{Vieillard-2020DCPI} but improved for the continuous AC setting.
%To the best of authors' knowledge, this is the first
%CAC is an off-policy algorithm that is capable of cautiously learning with guarantee against adversarial situations and various sources of error \citep{kozunoCVI,vieillard2020leverage}.
%Specifically, the CAC actor exploits linear interpolated policy to achieve monotonic improvement and the critic is regularized by Shannon entropy and KL divergence following the recent booming of relevant ADP literature \citep{pmlr-v80-haarnoja18b,vieillard2020leverage} that has demonstrated state-of-the-art performance.
%Specifically, both the actor and the critic of CAC are regularized by Shannon entropy and KL divergence following the recent booming of ADP literature that has demonstrated state-of-the-art performance. 
%The interpolation coefficient is automatically adjusted. We also compare our coefficient design with that of a very recent work \citep{Vieillard-2020DCPI} to illustrate the advantage of our method.

The rest of the paper is organized as follows: we provide a short survey on related work in Section \ref{sec:related_work}, followed by the preliminary on RL and relevant AC methods in Section \ref{sec:backrgound}.
Section \ref{sec:cac} presents CAC. Specifically, we discuss our novel design of the interpolation scheme which is central to CAC.
%We describe how entropy regularization helps to tackle several difficulties under continuous action domains which is closely related to our contribution.
We evaluate CAC in Section \ref{sec:experiment} on a set of challenging continuous control problems. We show that CAC is capable of achieving performance comparable to the state-of-the-art AC methods while significantly stabilizing learning. Ablation study has been conducted to distinguish CAC from existing methods. Discussion and conclusion are given in Section \ref{sec:conclusion}. 
Due to page limits, we present derivations and implementation details in the Appendix of supplemental file.

\section{Related Work}\label{sec:related_work}

%\subsection{Actor-Critic Methods}\label{sec:related_ac}

%AC methods have long been used for tackling RL problems \citep{Konda00,BHATNAGAR2009}. It derived its name from that the policy \emph{actor} is guided by the \emph{critic} that assesses the actor's performance via value functions.
%As the actor learns via policy gradient \citep{sutton2000policy}, it can naturally handle problems with continuous action spaces \citep{lillicrap2015continuous} while still enjoying the guidance of the value functions. 
It has been noticed that various sources of error such as approximation error in AC algorithms \citep{pmlr-v80-fujimoto18a,pmlr-v97-fu19a} are the cause of performance degradation and oscillation during learning. 
In this section, we briefly survey some related works (partially) tackling this problem and outline our contributions.

\textbf{Robust AC. } 
Algorithms learning from off-policy data are sample-efficient but are also at the risk of divergence. 
To solve the divergence problem, an approach is to incorporate the importance sampling (IS) ratio \citep{precup2001-offpolicy}.
However, the resultant algorithms typically have large variance as the IS ratio is the product of many potentially unbounded terms. 
\cite{munos2016safe} proposed to clip the IS ratio and proved the resulting algorithm Retrace $\!(\lambda)$ can attain the globally optimal policy. 
Retrace ($\lambda$) has motivated recent successful AC methods \citep{wang2017sample,fakoor2020p3o} that exploit both on- and off-policy samples for better stability while retaining sample-efficiency. 
The robustness comes from that any off-policy samples can be used without causing divergence and wild variance thanks to the clipping. 
However, one still has to trade off the learning speed with learning stability by the user-defined clipping threshold. If we favor more learning stability, the agent might fail to learn meaningful behaviors.

\textbf{Entropy-regularized AC. }
The recently booming entropy-regularized ADP literature has established that by augmenting the reward with Shannon entropy, the optimal policy is multi-modal and hence robust against adversarial settings \citep{Nachum2017-bridgeGap,haarnoja-2017a,ahmed19-entropy-policyOptimization}.
Another popular candidate entropy is the relative entropy or Kullback-Leibler (KL) divergence that renders the optimal policy an average of all past value functions \citep{azar2012dynamic,Fox2016,kozunoCVI,vieillard2020leverage}, which is more conservative and robust under mild assumptions such as the sequence of errors is martingale difference under the natural filtration \citep{azar2012dynamic}. 
These methods have been extended to the AC setting including state-of-the-art \citep{pmlr-v80-haarnoja18b} that exploits the Shannon entropy and \citep{nachum2017trust} that leverages the KL divergence. 
Those methods demonstrate strong empirical performance on a wide range of tasks. However, it should be noted that the conservativeness brought by the entropy-regularized reward augmentation exists only in the critic.
Since the average is prone to outliers, performance degradation can still happen if the value function estimates at some iterations are poor.
%Especially the extensions \citep{nachum2017trust,pmlr-v80-haarnoja18b} of the recently popular entropy-regularized ADP algorithms \citep{azar2012dynamic,Nachum2017-bridgeGap,haarnoja-2017a}.
%However, AC methods are typically prone to oscillate in performance.
%Actor-critic methods have seen a resurgence of interest due to the successful combination of deep networks. 
%Deep Deterministic Policy Gradient (DPPG) \citep{lillicrap2015continuous} first successfully extended DQN to continuous action spaces. However, DDPG typically oscillates wildly during learning. 
%To achieve stable learning as well as sample efficiency, recent successful AC methods \citep{trpo-schulman15,schulman1707proximal,wang2017sample,pmlr-v80-haarnoja18b,fakoor2020p3o} take inspirations from the booming ADP literature on conservative learning \citep{kakade-cpi,munos2016safe,Vieillard-2020DCPI} and entropy-regularized policy iteration (PI) \citep{kozunoCVI,vieillard2020leverage}. 
%Taking the monotonic improvement concept from \citep{kakade-cpi}, the on-policy algorithms trust region policy optimization (TRPO) \citep{trpo-schulman15} and the proximal policy optimization (PPO) \citep{schulman1707proximal} achieved more stable performance. TRPO can be viewed as a member of entropy-regularized RL family, a recent prominent member of which is the soft actor-critic (SAC) that maximizes also entropy of policies besides task reward. Since SAC is off-policy, learning performance also oscillates frequently. 

\textbf{Conservative Policy Iteration. }
Tackling the performance oscillation problem has been widely discussed in the ADP literature \citep{wagner2011,Bertsekas2011}, of which the seminal CPI algorithm \citep{kakade-cpi} has inspired many conservative learning schemes with strong theoretical guarantees for per-update improvement \citep{pmlr-v28-pirotta13,abbasi-improvement16}. 
However, CPI has seen limited applications to the AC setting due to two main drawbacks: (1) it assumes a \emph{good} reference policy that is typically difficult to obtain in high-dimensional continuous state-action spaces; (2) the interpolation coefficient that interpolates the reference policy and current policy depends on the horizon of learning, which is typically short in ADP scenarios. In the AC setting featuring long learning horizon, this coefficient becomes vanishingly small and hence significantly hinders learning.
A very recent work extended CPI to learning with deep networks and has demonstrated good performance on Atari games \citep{Vieillard-2020DCPI}. However, it is limited to discrete action spaces while our method mainly focuses on continuous action spaces and can be easily adapted to discrete action setting.
%CPI alleviates policy degradation and oscillation by leveraging interpolated policies.
%This is mainly because CPI interpolates the current policy with some \emph{reference policy} for every policy update. 
The above-mentioned drawbacks render CPI generally perceived as unwieldy \citep{trpo-schulman15}.

\textbf{Trust-region Methods. }
Motivated by the above-mentioned drawbacks of CPI, two popular variants trust region policy optimization (TRPO) \citep{trpo-schulman15} and its improved version proximal policy optimization (PPO) \citep{schulman1707proximal} sidestep the interpolation and directly approximate the resultant conservative policy.
TRPO and PPO are welcomed choices for learning from scratch when the reference policy is unavailable or unreliable, but they also ignore this knowledge when we have a good reference policy at our disposal.
%Hence they do not leverage the information from the reference policy. 
Further, TRPO and PPO require on-policy samples which are expensive since all samples can be used only once and then discarded.

\textbf{Contribution. }
The main contributions of this paper are:
\begin{itemize}
    \item CAC, the first off-policy AC method applying the interpolation of CPI-based algorithms for stabilizing off-policy learning to the best of the authors' knowledge.
    \item A novel interpolation coefficient design suitable for high dimensional continuous state-action spaces. Previously there was only a design suitable for discrete action spaces \citep{Vieillard-2020DCPI}.
    \item We evaluate CAC on a set of benchmark continuous control problems and demonstrate that CAC achieves comparable performance with state-of-the-art AC methods while significantly stabilizes learning. 
\end{itemize}

\section{Preliminary}\label{sec:backrgound}

\subsection{Reinforcement Learning }

RL problems are mostly formulated by Markov Decision Processes (MDPs) defined by the tuple $(\mathcal{S}, \action, P, \reward, \gamma)$, where $\state$ is the state space, $\action$ is the (possibly continuous) action space,  $\trans$ is the transition probability from $s$ to $s'$ under action $a$ taken; $\reward$ is the reward function with $\reward(s,a) \in [-r_{\text{max}}, r_{\text{max}}]$ denoting the immediate reward associated with that transition. 
We also use $r_{t}$ as a shorthand for $\mathcal{R}(s_{t}, a_{t})$ at $t$-th step. $\gamma \!\in\! (0, 1)$ is the discount factor.
%We use $r_t$ to denote a reward at time $t$: $r_t=\reward(\st, \at)$.
We define $\mathbf{P}$ as a left operator such that $(\mathbf{P}V)(s,a) \!=\! \sum_{s'}\!\trans \!V(s')$ for some $V$.
A policy $\pi$ maps states to a probability distribution over the action space. 
We define the stationary state distribution induced by $\pi$ as the unnormalized occupancy measure $d^{\pi}(s) = \sum_{t=0}^{\infty} \gamma^{t} P\left(\st=\sv \mid \pi\right)$. 
The goal of RL is to find an optimal policy $\pi^{*}$ that maximizes the long term discounted rewards $\return^{\pi} = \avg{d^{\pi}} \left[\sum_{t=0}^{\infty} \gamma^{t} r_t\right]$.
Equivalently, this optimal policy also maximizes the state value function for all states $s$:
\begin{align*}
  \pi^{*} = \argmax_{\pi} V^{\pi}(s) = \argmax_{\pi}\mathbb{E}^{d^{\pi}}\left[\sum_{t=0}^{\infty}\gamma^{t} r_{t} \big| s_{0} \!=\! s \right].
\end{align*}
%where the expectation is with respect to the transition dynamics $\mathcal{T}$ and policy $\pi$. 
The state-action value counterpart $Q^{\pi^{*}}$ is more frequently used in the control context:
\begin{align*}
  Q^{\pi^{*}}(s, a) = \max_{\pi} \mathbb{E}^{d^{\pi}}\left[\sum_{t=0}^{\infty}\gamma^{t} r_{t} \big| s_{0} \!=\! s, a_{0} \!=\! a \right].
\end{align*}
We define the advantage function for a policy $\pi$ as $A^{\pi}(s,a) = Q^{\pi}(s,a) - V^{\pi}(s)$, and the expected advantage function of policy $\pi'$ over $\pi$ as:
\begin{align*}
    A^{\pi'}_{\pi}(s) \!=\! \avg{\pi'}\adaRectBracket{Q^{\pi}(s,a) } - V^{\pi}(s).
\end{align*}

\subsection{Actor Critic methods}\label{sec:ac}

%Actor-critic methods have long been used for tackling RL problems \citep{Konda00,BHATNAGAR2009}. It derived its name from that the policy \emph{actor}, which is typically parametrized by neural networks, generates samples evaluated by the values estimates of the critic. As the actor learns via policy gradient \citep{sutton2000policy}, it can naturally handle problems with continuous action spaces while still enjoying the guidance of the value functions.
In this section we briefly introduce recent actor-critic algorithms and discuss their pros and cons and shed light on our proposal in Section \ref{sec:cac}.

%Actor-critic is an approach to solve RL problems where the parametrized policy (actor) $\pol_\pparams$ learns by using feedback from the action-value function (critic) $Q_\vparams$.
%In this section, to highlight the differences between the previous actor-critic framework and ours for later discussions, we summarize actor-critic methods with good performance and sample efficiency, 1.~off-policy entropy-regularized actor-critic (e.g. SAC) and 2.~on- and off-policy actor-critic (e.g. P3O).
%We skip fully on-policy actor-critic (e.g. TRPO and PPO) since it can be generalized by on- and off-policy actor-critic methods.

\subsubsection{Trust Region Methods}\label{sec:trust_region}

%(Introduction of TRPO goes here. Specifically, TRPO follows from Kakade's inequality but gets away from the $\zeta$. Namely, TRPO tries to directly approximate the policy that would be produced by interpolation. Hence we shall state that, having $\zeta$ is more flexible as it allows us to directly span a linear policy class. The inconvenience of $\zeta$ can be avoided by leveraging information projection $\argmin KL$, which has a theorem stating that the policy after projection has the same state-action occurrence with the interpolated policy.) 
%TRPO \citep{trpo-schulman15} proposed to get rid of $\zeta$ as mixture policies are unwieldy to use. Hence TRPO can be interpreted as directly approximating the monotonically improving policy $\tilde{\pi}_{k+1}$ without resorting to the linear class spanned by $\pi'$ and $\pi_{k}$. %TRPO was further improved by proximal policy gradient (PPO) \citep{schulman1707proximal} that simplified both theory and implementation.

TRPO exploits the policy improvement lemma of \citep{kakade-cpi} for ensuring approximately monotonic policy improvement. However, unlike in \citep{kakade-cpi} that at $k$-th iteration the policy is updated as $\pi_{k+1} = \zeta\pi'+(1-\zeta)\pi_{k}$, where $\pi'$ is the greedy policy; TRPO constructs an algorithm that directly computes $\pi_{k+1}$ without resorting to $\pi'$. Specifically, TRPO has the following update rule:
\begin{align}
    \begin{split}
        &J^{\text{TRPO}}_{\pi_{k}}(\pi) := \, \argmax_{\pi} \avg{\pi, d^{\pi_{k}}}\adaRectBracket{A^{\pi_{k}}}, \\
        &\text{subject to } \quad  C_{\gamma} \,\Delta_{\pi}\, D^{\text{max}}_{KL}(\pi_{k} |\!| \pi) \leq \delta, \\
        &\text{with } \,\,  \Delta_{\pi} = \max_{s,a} |A^{\pi}(s,a)|,
    \end{split}
    \label{eq:trpo}
\end{align}
%where $C$ is a function of the maximum advantage $\max_{s,a} A^{\pi_{k}}(s,a)$ and 
where $C_{\gamma}$ is a horizon-dependent constant, $D_{KL}$ is the KL divergence and $\delta$ is the trust region parameter.

As computing $J^{\text{TRPO}}_{\pi_{k}}(\pi)$ requires sampling according to the stationary distribution $d^{\pi_{k}}$, it is inherently an on-policy algorithm, which is not sample-efficient as the samples can only be used only once and discarded.

\subsubsection{Off-policy Maximum Entropy Actor-Critic}\label{sec:entrl}

As state-of-the-art model-free off-policy AC algorithm, soft actor-critic (SAC) \citep{pmlr-v80-haarnoja18b} maximizes not only task reward but also the Shannon entropy of policy. The entropy term in the reward function renders the optimal policy multi-modal as opposed to deterministic policies of algorithms that solely maximize task reward, which is beneficial due to the multi-modality \citep{haarnoja-2017a} and has demonstrated superior sample-efficiency due to more effective exploration of the state-action spaces. Writing in the ADP manner, SAC has the following update rule (we drop the state-action pair for the $Q$ function for simplicity):
\begin{align}
    \begin{split}
        &\begin{cases}
        % \avg{ d^{\pi}}\left[ \sum_{t\geq 0} \reward(s_{t}, a_{t}) 
            \pi_{k+1} \leftarrow \argmax_{\pi} \avg{\pi}\adaRectBracket{ Q^{\pi_{k}}_{\mathcal{H}}(s,a) + \kappa \mathcal{H}\left(\pi(\cdot|s_{t})\right)} \\
            Q^{\pi_{k+1}}_{\mathcal{H}} \leftarrow \rewardNoT + \gamma\adaBracket{\mathbf{P}V^{\pi_{k}}_{\mathcal{H}}}\!(s,a) 
        \end{cases}\\
        & \text{where } V^{\pi}_{\mathcal{H}}(s) = \sum_{t\geq 0}\gamma^{t} \avg{\pi}\left[ \reward(s_{t}, a_{t}) + \kappa \mathcal{H}\left(\pi(\cdot|s_{t})\right) \Big| s_{0} = s\right].
    \end{split}
    \label{eq:sac}
\end{align}
$\mathcal{H}(\pi):= -\sum_{a} \pi(a|s)\log\pi(a|s)$ denotes the Shannon entropy of policy $\pi$, $\kappa$ denotes the weighting coefficient and $V^{\pi}_{\mathcal{H}}$ denotes the soft value function when regularized with the Shannon entropy.
SAC performs one step look-ahead for updating the actor, where states are randomly sampled from a replay buffer, and then actions are generated by feeding the states into the policy network \citep{pmlr-v80-haarnoja18b}. As such, SAC does not need an IS ratio, but it has been demonstrated that SAC often oscillates wildly in performance.

\section{Cautious Actor Critic}\label{sec:cac}

%In this section, we will present our proposed reinforcement learning algorithm, cautious actor-critic (CAC).
%CAC aim to overcome the shortcomings of the previous actor-critic methods, and achieve 1. monotonic policy improvement, 2. sample efficiency, and 3. the state-of-the-art performance.
%The core principle of CAC is to maximize the entropy regularized return while stabilizing policy update through stochastic mixtures of consecutive policies.

In this section we present CAC, an off-policy actor-critic method capable of learning conservatively against performance oscillation and degradation. %The name \emph{cautious} stems from the doubly conservativeness from \emph{conservative policy iteration} \citep{kakade-cpi} and entropy-regularized \emph{conservative value iteration} (CVI) \citep{kozunoCVI}. 

%In this section we detail the structure of CAC in Section \ref{sec:cac_alg} and discuss the design of $\zeta$ suitable for learning with deep networks and compare it with \citep{Vieillard-2020DCPI} in Section \ref{sec:zeta_design}.

\subsection{CAC Algorithm}\label{sec:deep_cac}

For the ease of understanding, we write CAC in the following approximate policy iteration style \citep{vieillard2020leverage}. 
Specifically, the first step corresponds to the policy (actor) improvement and the last step corresponds to the interpolation:
\begin{align}
\begin{split}
& \text{CAC} \,\, 
\begin{cases}
    \pi_{k+1} \leftarrow \argmax_{\pi} \avg{\pi}\!\adaRectBracket{ Q^{\pi_{k}}_{\mathcal{I}}(s,a) + \entreg{\pi_{k}}{\pi}(s) }\\
    Q^{\pi_{k+1}}_{\mathcal{I}} \leftarrow \reward(s,a) + \gamma\adaBracket{\textbf{P}V^{\pi_{k+1}}_{ \mathcal{I}}}\!(s,a) \\   
    \zeta \leftarrow  \, (\tilde{\Delta}^{\pi_{k+1}}_{\pi_{k}})^{-1}\! \adaBracket{\avg{\pi_{k+1}, \onrbuf} \adaRectBracket{A^{\pi_{k}}(s,a)} } \\
    %\zeta \leftarrow \avg{\pi_{k+1}, d^{\pi_{k}}} \adaRectBracket{A^{\pi_{k}}(s,a)} \\
    % \propto (\tilde{\Delta}^{\pi_{k+1}}_{\pi_{k}})^{-1}\avg{\pi_{k+1}, d^{\pi_{k}}} \adaRectBracket{A^{\pi_{k}}_{\mathcal{I}}(s,a)} 
    \tilde{\pi}_{k+1} \leftarrow  \zeta \pi_{k+1}  + (1 - \zeta)\pi_{k} 
\end{cases}\\
&\text{with } V^{\pi_{k+1}}_{\mathcal{I}}(s) = \sum_{t\geq 0}\gamma^{t} \avg{\pi}\left[ \reward(s_{t}, a_{t}) + \entreg{\pi_{k}}{\pi_{k+1}}(s_{t}) \,\Big| \, s_{0} = s\right],\\
&\entreg{\pi_{k}}{\pi_{k+1}}(s) = \avg{\pi_{k+1}} \! \adaRectBracket{-\kappa\log\pi_{k+1}(a|s) - \tau \log\frac{\pi_{k+1}(a|s)}{\pi_{k}(a|s)}},
\end{split}\label{eq:cac}
\end{align}
%(16\gamma C_{K})^{-1} (1-\gamma)^{3}
%$ \zeta \!=\! C  \, \avg{\pi_{k+1}, d^{\pi_{k}}} \adaRectBracket{A^{\pi_{k}}(s,a)} $ and $C$ is a constant (see Appendix for details).
where $\zeta$ is the interpolation coefficient computed by $\zeta^{*}$ in \EqNumber{eq:cac_implement},  $\onrbuf$ denotes the on-policy replay buffer, with $K$ indicating the number of steps up to now.
$\kappa, \tau$  denote the Shannon entropy and KL divergence regularization coefficient, respectively.
$Q_{\mathcal{I}},V_{\mathcal{I}}$ (and hence $A_{\mathcal{I}}$) denote the entropy-regularized value functions. For uncluttered notations, in the rest of the paper we drop the subscript $\mathcal{I}$.
Except the computation of $\zeta$, all other steps are computed using samples from the off-policy replay buffer $\rbuf$.
For later convenience, we define $\alpha := \frac{\kappa}{\kappa + \tau}$ and $\beta := \frac{1}{\kappa + \tau}$. 
Note our use of both on- and off-policy replay buffers renders CAC similar in spirit to \citep{gu2017interpolated,wang2017sample,fakoor2020p3o}.
%It can be shown that CAC converges to the optimum in the tabular case, we defer the details to Section \ref{apdx:ER-conservative-policy-iteration} in the Appendix.

We first compute the greedy policy $\pi_{k+1}$. Due to the Fenchel conjugacy \citep{geist19-regularized}, when $\entreg{\pi_{k}}{\pi_{k+1}}$ is included in the $\argmax$, the maximizer policy can be analytically derived as $\pi_{k+1}(a|s) \propto \- \pi^{\alpha}_{k}(a|s)\exp\adaBracket{\beta Q_{\pi_{k}}(a|s)}$ \citep{kozunoCVI}. 
Then the entropy-regularized action value function $Q^{\pi_{k+1}}$ is evaluated.    
In the third step, we use the on-policy replay buffer $\onrbuf$  for computing $\zeta$ as described in \EqNumber{eq:cac_implement}. 
Finally, the optimal policy in the sense of guaranteeing policy improvement is obtained by interpolating $\pi_{k+1}$ with $\pi_{k}$.
We present the following theorem of CAC convergence in the tabular case.
%We now discuss in detail the conservative actor and the conservative critic.

\begin{theorem}
Repeated application of CAC Eq. (\ref{eq:cac}) on any initial policy $\pi$ will make it converges to the entropy regularized optimal policy $\pol^{*}(a|s) = \frac{\exp\adaBracket{\frac{1}{\kappa} Q^{*}(s,a)}}{\int_{a\in\mathcal{A}} \exp\adaBracket{\frac{1}{\kappa} Q^{*}(s, a) } }$.
\end{theorem}

\begin{proof}
See Section \ref{apdx:ER-conservative-policy-iteration} in the Appendix.
\end{proof}
%\vspace{-0.35cm}
From Theorem 1 we see the optimal policy and corresponding optimum of the MDP is biased by the choice of $\kappa$. If we gradually decay the value of $\kappa$ then we recover the optimum of the non-regularized MDP \cite{vieillard2020leverage}.
In the following sections, we describe in detail the CAC actor and critic, as well as the derivation of actor gradient expression and practical interpolation coefficient $\zeta$ design. 

\subsubsection{Conservative Actor}

As discussed in Section \ref{sec:trust_region}, at $k$-th iteration TRPO directly constructs a new policy $\pi$ by maximizing $J^{\text{TRPO}}_{\pi_{k}}(\pi)$. This is useful if the agent learns from scratch, but it discards the \emph{reference policy} when available. 
On the other hand, we follow the exact form of \citep{kakade-cpi} by taking the information of reference policy into account, where we choose $\pi_{k+1}$ to be the reference policy $\pi'$:
\begin{align}
    \tilde{\pi}_{k+1} = \zeta \pi_{k+1} + (1 - \zeta) \pi_{k}.
    \label{eq:interpolation}
\end{align}
%where $\zeta$ is the interpolation coefficient. 
Our objective function $J_{\pi_{k}, \pi_{k+1}}^{\text{CAC}}(\pi)$ explicitly features the knowledge of the reference policy \citep{pmlr-v28-pirotta13}. Specifically, the objective can be lower-bounded as:
\begin{align}
    \begin{split}    \label{eq:Jimproved2}
          J_{\pi_{k}, \pi_{k+1}}^{\text{CAC}}(\pi) &:= \avg{\pi_{k+1}, d^{\pi_{k+1}}}\adaRectBracket{A^{\pi_{k}}(s,a)} \\
            &\geq  C'_{\gamma} \,\,  (v\, \tilde{\Delta}^{\pi_{k+1}}_{\pi_{k}})^{-1} \adaBracket{\avg{\pi_{k+1}, d^{\pi_{k}}} \adaRectBracket{A^{\pi_{k}}(s,a)} }^{2}, \\
        %C' &= \frac{(1-\gamma)^{2}}{4 \gamma \tilde{\Delta}^{\pi_{k+1}}_{\pi_{k}}  C_{K}^{\frac{1}{2}}}, \\ 
        %\frac{(1-\gamma)^2}{2\gamma C_{K}^{\frac{1}{2}}}
        \text{given } \zeta^{*} &=   2 \, C'_{\gamma} (v\, \tilde{\Delta}^{\pi_{k+1}}_{\pi_{k}})^{-1} \adaBracket{\avg{\pi_{k+1}, d^{\pi_{k}}} \adaRectBracket{A^{\pi_{k}}(s,a)} \!}, \\
        \tilde{\Delta}^{\pi_{k+1}}_{\pi_{k}} &= \max_{s,s'} \left|A^{\pi_{k+1}}_{\pi_{k}}(s) - A^{\pi_{k+1}}_{\pi_{k}}(s') \right|, \\
        v &= \max_{s} D_{TV}\adaBracket{\pi_{k+1}(\cdot|s) |\!| \pi_{k}(\cdot|s)},
        \vspace{-0.2cm}
    \end{split}
\end{align}
where $D_{TV}$ denotes the total variation. $v$, $\tilde{\Delta}^{\pi_{k+1}}_{\pi_{k}}$ and the expectation wrt $\pi_{k+1}, d^{\pi_{k}}$ require estimation. $C'_{\gamma}$ absorbs the horizon-dependent constants. 
  %It is obvious that $\avg{\pi_{k+1}, d^{\pi_{k}}} \adaRectBracket{A^{\pi_{k}}(s,a)}$ is the only quantity needs estimation. 
Hence, when optimizing the lower bound of $J_{\pi_{k}, \pi'}^{\text{CAC}}(\pi)$, we can achieve guaranteed improvement.

In the existing literature \citep{kakade-cpi,pmlr-v28-pirotta13,abbasi-improvement16}, the difficulties of extending \EqNumber{eq:Jimproved2} to large-scale problems are: 
(1) preparing a reliable reference policy in high dimensional continuous state-action spaces is difficult; 
(2) it is hard to accurately estimate $v$, the maximum total variation between two policies without enforcing a gradual change of policies, which is absent in these works. 
On the other hand, naively using $v \leq 2$ as suggested by \citep{pmlr-v28-pirotta13} often yields vanishingly small $\zeta$, which significantly hinders learning. 
(3) the horizon-dependent constant $C_{\gamma}'$ developed in the classic ADP literature is not suitable for learning with deep networks that feature long horizon of learning. 
As will be demonstrated in the following sections, we tackle the first and second problems by leveraging entropy-regularized critic, and the third problem via a novel design of $\zeta$ inspired by a very recent work for discrete action problems \citep{Vieillard-2020DCPI}.

%When run in policy iteration manner, this algorithm converges with the rate $O\!\adaBracket{\exp\adaBracket{-(1-\gamma)\sum_{j=1}^{k}\zeta_{j}}}$ \citep{Vieillard-2020DCPI}, which is slower than the rate $O(\gamma^{k})$ of standard approximate value iteration (AVI) \citep{scherrer15-AMPI}. This is expected as the algorithm trades off slower learning for conservative behaviors.
  
%It is worth noting that \citep{Vieillard-2020DCPI} adopt general stationary policy $\pi_{k+1}$ instead of entropy-regularized one. They propose to use off-policy samples from replay buffer for estimating the expectation, which might corrupt the policy improvement guarantee. On the other hand, we exploit  on-policy samples for this estimation.  
%In the experimental section we also compare with their proposed. 
%It is worth noting that as we use on-policy replay buffer $\onrbuf$ for estimating $\zeta$ and off-policy samples for evaluating and updating $\pi_{k+1}$, our method of using both on- and off-policy data is in spirit similar to \citep{gu2017interpolated,wang2017sample,fakoor2020p3o}.

\subsubsection{Conservative Critic }  

%Entropy-regularization is of central importance to the recent state-of-the-art RL algorithms \citep{haarnoja-2017a,pmlr-v80-haarnoja18b,vieillard2020leverage}. 
%In this study, we consider regularization with the extensively studied Shannon entropy and KL divergence.

In \EqNumber{eq:Jimproved2} we see in order to yield a meaningful interpolation coefficient $\zeta$ one is required to accurately estimate the maximum total derivation $v$, which is intractable in high dimensional continuous action spaces. 
However, by introducing an entropy-regularized critic, we can leverage the following theorem to avoid estimating $v$:
\begin{theorem}\citep[Proposition 3]{kozunoCVI}
\label{theorem:ck}
For any two consecutive entropy-regularized policies $\pi_{k}, \pi_{k+1}$ generated by \EqNumber{eq:cac}, the following bound for their maximum total deviation holds:
\begin{equation}
\begin{aligned}
     &\max_{s} D_{TV}\adaBracket{ \pi_{k+1} (\cdot | s)|\!| \pi_{k} (\cdot | s)} \leq  \sqrt{4 B_{K} + 2 C_{K}}, \\
     &\text{where }  B_{k}:=\frac{1-\gamma^{k}}{1-\gamma}\epsilon\beta, \quad C_{k} : = r_{\text{max}} \beta \sumJK{0}{k-1}\alpha^{j}\gamma^{k-j-1},
\end{aligned}\label{eq:ck}
\end{equation}
$\epsilon$ is the uniform upper bound of errors.
%Recall that $\alpha$ and $\beta$ are defined by the entropy bonus and KL regularization coefficients $\kappa$  and $\tau$ as $\alpha := \frac{\kappa}{\kappa + \tau}$ and $\beta := \frac{1}{\kappa + \tau}$.
\end{theorem}
Recall from Section \ref{sec:deep_cac} that $\alpha:=\frac{\kappa}{\kappa + \tau}, \beta:=\frac{1}{\kappa + \tau}$.
%For simplicity of later development, we assume there is no error. But it is straightforward to incorporate it.
Specifically, it has been proved in \citep{kozunoCVI} that this bound is non-improvable, i.e. there exists an MDP such that the inequality becomes equality.

By leveraging an entropy-regularized critic, the objective in \EqNumber{eq:Jimproved2} becomes:
\begin{align}
    \begin{split}    \label{eq:cac_implement}
          J_{\pi_{k}, \pi_{k+1}}^{\text{CAC}}(\pi) &\geq  C'_{\gamma}C_{k} \,\,  ( \tilde{\Delta}^{\pi_{k+1}}_{\pi_{k}})^{-1} \adaBracket{\avg{\pi_{k+1}, d^{\pi_{k}}} \adaRectBracket{A^{\pi_{k}}(s,a)} }^{2}, \\
        %C' &= \frac{(1-\gamma)^{2}}{4 \gamma \tilde{\Delta}^{\pi_{k+1}}_{\pi_{k}}  C_{K}^{\frac{1}{2}}}, \\ 
        %\frac{(1-\gamma)^2}{2\gamma C_{K}^{\frac{1}{2}}}
        \text{given } \zeta^{*} &=   2 \, C'_{\gamma}C_{k} ( \tilde{\Delta}^{\pi_{k+1}}_{\pi_{k}})^{-1} \adaBracket{\avg{\pi_{k+1}, d^{\pi_{k}}} \adaRectBracket{A^{\pi_{k}}(s,a)} \!}, 
    \end{split}
\end{align}
where $C'_{\gamma}$ absorbs horizon-dependent constants and $C_{k}$ is from Theorem \ref{theorem:ck}.
Estimating the optimal $\zeta^{*}$ now requires estimating the expectation wrt $\pi_{k+1}, d^{\pi_{k}}$ and $\tilde{\Delta}^{\pi_{k+1}}_{\pi_{k}}$ which have been studied by \citep{Vieillard-2020DCPI}.

The KL divergence also manifests its importance for generating reasonable reference policies even for high dimensional or continuous action problems. 
Consider the following upper bound due to \citep{vieillard2020leverage} where reward is augmented by the KL divergence: 
\begin{align*}
        \myNorm{Q^{*} - Q^{\pi_{k+1}}} \leq \frac{2}{1-\gamma}\myNorm{\frac{1}{k}\sum_{j=0}^{k}\epsilon_{j}} \!\!+ \frac{4}{1-\gamma} \frac{V_{max}}{k},
\end{align*}
where $\epsilon_{j}$ are errors and $V_{max} \!=\! \frac{r_{\text{max}}}{1-\gamma}$. By comparing it with the non-improvable approximate modified policy iteration (AMPI) bound where the reward is not augmented \citep{scherrer15-AMPI}:
\begin{align*}
    \myNorm{Q^{*} - Q^{\pi_{k+1}}} \leq \adaBracket{\! (1-\gamma) \sumJ{1}\myNorm{\epsilon_{j}} \!} + \frac{2\gamma^{k+1}}{1-\gamma}V_{max},
\end{align*}
we see that the error term for the KL regularization case is sup-over-sum. Under mild assumptions such as $\epsilon_{j}$  are iid distributed under the natural filtration \citep{azar2012dynamic}, the summation over errors asymptotically cancels out. 
On the other hand, the error term for AMPI depends on the summation of \emph{maximum} of every iteration, which is typically large. 
Further, the dependence of error on the horizon is linear $\frac{1}{1-\gamma}$ rather than quadratic, which is a significant improvement as typically $\gamma \approx 1$.

%On the other hand, though theoretically it is not well understood the benefits of introducing the Shannon entropy $\entropy{\pi}$ into the reward \citep{vieillard2020leverage}, it has been empirically verified extensively that entropy helps exploration, renders the optimal policy multi-modal \citep{haarnoja-2017a,pmlr-v80-haarnoja18b} and smooths the optimization landscape \citep{ahmed19-entropy-policyOptimization}. It should be noted that, adding entropy to the reward will shifts the optimality of the MDP \citep{geist19-regularized}, but it is often regarded worthwhile for trading off better exploration \citep{mnih2016asynchronous}.

\subsubsection{Network Optimization Perspective } 

Given the above ADP-style characterization for both the actor and the critic, we now examine \EqNumber{eq:cac} from the optimization perspective. 
Suppose the critic is parametrized by a network with parameters $\theta$ and the actor by a network with parameters $\phi$.
CAC updates the network weights $\theta, \phi$ by solving the following minimization problems:
\begin{subequations}
\begin{align}
y &= r+\gamma \left(\avg{\av \sim \pi_\phi}\left[Q_{\bar{\theta}} (\svp, \av)\right] + \entreg{\pol_{\bar{\pparams}}}{\pol_{\pparams}}(\svp)\right), \label{eq:y_target}\\
\vparams &\leftarrow \arg \min \avg{\rbuf}\left[\left({Q}_{\theta}\left(\sv , \av \right) - y\right)^{2}\right], \label{eq:theta}\\
{\pparams}&\leftarrow \arg \min \avg{\rbuf}\left[{D}_{{KL}}\left(\pi_{\phi} \|
(1-\zeta)\pol_{\bar{\pparams}} + \zeta \gr_{\pol_{\bar{\pparams}}}Q_{\vparams} \right)\right] \label{eq:phi},
\end{align}
\end{subequations}
where the update of $\phi$ corresponds to solving an \emph{information projection} problem.
This is because policies $\pi_{k+1}, \pi_{k}$ are Boltzmann softmax \citep{geist19-regularized} but their summation is generally not Boltzmann, which might results in loss of desirable properties. 
By the following theorem, in the ideal case we can find a Boltzmann policy $\pi_{\phi}$ that perfectly represents the interpolation by solving the information projection problem. 
\begin{theorem}\cite[Theorem 2.8]{ziebart2010-phd}\label{theorem:mixture}
Let $\pi^{(1)},$  $\pi^{(2)},$ $\dots, \pi^{(n)}$  be an arbitrary sequence of policies and $\zeta_{1}, \dots \zeta_{n}$ be a sequence of numbers such that $\zeta_{i} \geq 0, \forall i$, $\sum_{i=1}^{n} \zeta_{i} = 1$. Then the policy $\pi'$ defined by:
 \begin{align}\label{eq:mixture_policy}
     \pi'(a|s) := \frac{ \sum_{i=1}^{n} \zeta_{i} \, P(\mathcal{S} = s, \mathcal{A} = a | \pi^{(i)}) }{\sum_{i=1}^{n} \zeta_{i} \, P(\mathcal{S} = s | \pi^{(i)})}
 \end{align}
 has same expected number of state-action occurrences when the denominator is nonzero.
 \end{theorem}
 \vspace{-0.2cm}
 In implementation as the states are sampled from the replay buffer, there is an error term in this information projection step.
Taking the above information projection into account,  we elaborate upon gradient expression of the actor via the following proposition:

\begin{proposition}\label{thm:gradient}
Let the actor network be parametrized by weights $\phi$ and critic by $\theta$.
Define $\mathcal{G}Q_{\bar{\phi},\theta}$ as the greedy policy with respect to the CAC critic.
The subscript $\bar{\phi}$ comes from the baseline policy introduced by KL divergence.
Then the gradient of the actor update can be expressed as:
\begin{align}
\begin{split}
&\nabla_{\pparams} \avg{\substack{\sv\sim\rbuf\\\av\sim\pol_{\pparams}}}
\left[ D^{\phi}_{\bar{\phi}}
-\frac{\beta}{1+\const}Q_\theta(\sv, \av)
\right], \\
&\text{where } D^{\phi}_{\bar{\phi}} =  \log \pi_{\pparams}\left(\av  \mid \sv \right)
-\frac{\alpha+\const}{1+\const}\log\pi_{\bar{\pparams}}(\av; \sv ) \\
&\const = \frac{1-\zeta}{\zeta}\cdot
\frac{\pi_{\bar{\phi}}(\av\mid \sv)}{\gr Q_{\bar{\phi}, \theta}(\av \mid \sv)} .
\end{split}\label{eq:cac_sgd_grpi}
\end{align}

\end{proposition}

\begin{proof}
 See Section \ref{apdx:cac-gradient} in the Appendix.
\end{proof}

\vspace{-0.35cm}

This gradient expression is similar to SAC \citep{pmlr-v80-haarnoja18b} which is off-policy since states $s$ are sampled from the off-policy replay buffer $\rbuf$. However, CAC has the term $\log\pi_{\bar{\pparams}}(\av; \sv )$ from the KL regularization. The term $\const$ in both the $Q_{\theta}$ and $\log\pi_{\bar{\pparams}}(\av; \sv )$ involves $\zeta$ that encodes the information for guiding the gradient to \emph{cautiously} learn. 
%We prove in Section \ref{apdx:ER-conservative-policy-iteration} that in the tabular case CAC converges to the optimal entropy-regularized policy. This optimum has been characterized by \citep{geist19-regularized,vieillard2020leverage}.

\subsection{Design of Interpolation Coefficient}\label{sec:zeta_design}

One of the main difficulties to extending CPI to learning with deep networks is that $\zeta$ becomes vanishingly small due to the typically long horizon in the AC setting. To tackle this problem, \citep{Vieillard-2020DCPI} propose to heuristically design $\zeta$ to be a non-trivial value, which features the consideration of \emph{moving averages}.

\begin{figure*}[t]
    \vskip 0.2in
    \begin{center}
    \centerline{\includegraphics[width=1.0\columnwidth]{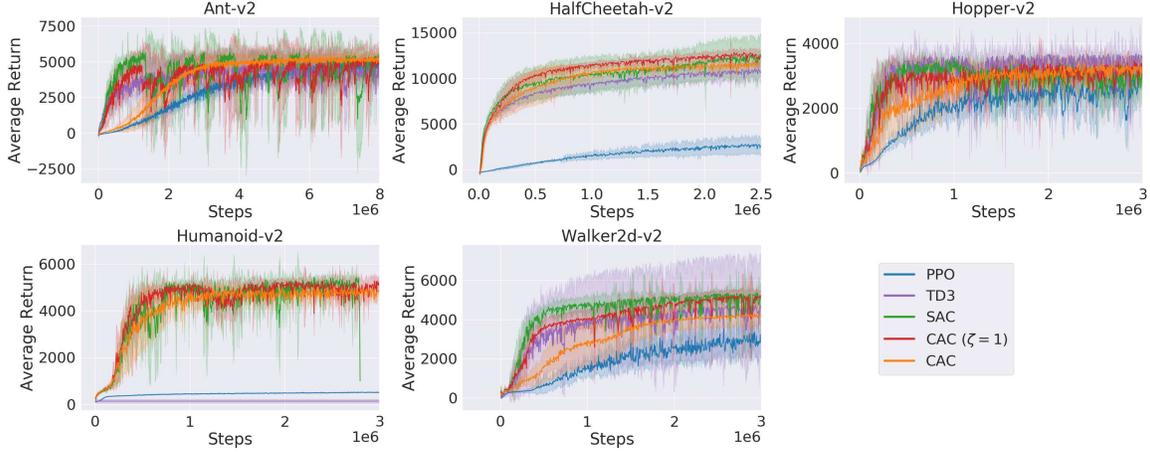}}
    \caption{Training curves on the continuous control benchmark problems. 
    The solid curves show the mean and shaded regions the standard deviation over the five independent trials.
    In all tasks CAC achieved comparable performance but significantly stabilized learning. 
    }
    \label{fig:mujoco-return}
    \end{center}
    \vskip -0.2in
    \end{figure*}

Recall that in \EqNumber{eq:cac_implement}, $\zeta$ is computed by a function of the form: 
\begin{align*}
    \zeta = C'_{\gamma}C_{k} \,\, \frac{\avg{\pi_{k+1}, d^{\pi_{k}}} \adaRectBracket{A^{\pi_{k}}(s,a)} }{ \tilde{\Delta}^{\pi_{k+1}}_{\pi_{k}} } ,
\end{align*}
where $C'_{\gamma}$ absorbs horizon-dependent constants and $C_{k}$ is defined in \EqNumber{eq:ck}. We propose to remove $C'_{\gamma}$ since it tends to zero as the horizon increases. Recall also that $\tilde{\Delta}^{\pi_{k+1}}_{\pi_{k}} $ is the maximum difference of the expected advantage function defined in \EqNumber{eq:Jimproved2}. We propose the following novel $\zeta$ design:
%Following this consideration and respecting \EqNumber{eq:Jimproved2}, we propose the following $\zeta$ design suitable for learning with deep networks:
\begin{equation} \label{eq:ERCAC-zeta}
\zeta^{\mathrm{CAC}}=\clip{\left(\frac{\approxadv}{|\approxadv^{\mathrm{MaxDiff}}|}, 0, \,\, 1\right)},
\end{equation}
where by following the moving average concept of \citep{Vieillard-2020DCPI}, we update $\approxadv$ and $\approxadv^{\mathrm{MaxDiff}}$ as:
\begin{equation}
\begin{array}{l} \label{eq:ERCAC-zeta2}
M = \avg{\substack{\sv\sim\onrbuf \\ \av\sim\gr Q^{\pi_k}}}\left[A^{\pi_k}(\sv, \av)\right] \\
\approxadv \leftarrow 
\begin{cases}
    c, & \text{if } M \leq 0\\
    \left(1-\lA\right) \approxadv + \lA M, & \text{else}
\end{cases} \\
\approxadv^{\mathrm{MaxDiff}} \leftarrow \left(1-\lAslow\right) \approxadv^{\mathrm{MaxDiff}}+ \lAslow \, M. 
\end{array}
\end{equation}
Here, $M$ is the current estimate of $\avg{\pi_{k+1}, d^{\pi_{k}}} \adaRectBracket{A^{\pi_{k}}(s,a)}$.
We propose to set an if-else judgement here, as $M < 0$ indicates the updated policy has worse performance than the current policy, we let $\approxadv$ be a negative value $c$, hence enforcing $\zeta = 0$.
This information is incorporated into $\approxadv$ by exponential moving average with the previous estimates. $\approxadv^\mathrm{MaxDiff}$ attempts to approximate the maximum difference $\tilde{\Delta}^{\pi_{k+1}}_{\pi_{k}}$.
$\onrbuf$ is a FIFO replay buffer storing $K$ on-policy  samples, and $\lA, \lAslow \in [0,1]$ are the hyperparameters controlling the average.
Following \citep{Vieillard-2020DCPI}, it is beneficial to have $\lAslow \leq \lA$ for smooth learning. 
%We describe \Eqref{eq:ERCAC-zeta} and \Eqref{eq:ERCAC-zeta2} step by step.

Computing $\zeta$ in \EqNumber{eq:ERCAC-zeta} using the moving average in \EqNumber{eq:ERCAC-zeta2} is in spirit similar to \citep{Vieillard-2020DCPI}. However, they focus on general stationary policies. As $\approxadv$ is an off-policy estimate of the on-policy term $\avg{\pi_{k+1}, d^{\pi_{k}}} \adaRectBracket{A^{\pi_{k}}(s,a)}$, it might corrupt the improvement guarantee. On the other hand, we focus on entropy-regularized policies which allow one to bound the performance loss of leveraging off-policy estimate $\approxadv$ \citep[Theorem 3]{zhu2020ensuring}.

\section{Experiments}\label{sec:experiment}

\begin{table*}[t]
\caption{The performance oscillation values of all algorithms for all environments. 
The bold numbers indicate the smallest performance oscillation values. 
$\times$ indicates the algorithm failed to learn meaningful behaviors. 
CAC recorded the smallest performance oscillation values for all the environments.
PPO is the only on-policy algorithm in the comparison.
}
\label{tb:oscillation}
\vskip 0.15in
\begin{center}
\begin{sc}
    \scalebox{0.925}{
\begin{tabular}{l|ccccc|ccccc}
\toprule
\textbf{}      & \multicolumn{5}{c}{\textbf{$\|\mathcal{O} J\|_{\infty}$}}    & \multicolumn{5}{c}{$\|\mathcal{O} J\|_{2}$} \\
\textbf{}      & \textbf{PPO} & \textbf{TD3} & \textbf{SAC} & \multicolumn{1}{c}{\begin{tabular}[c]{@{}c@{}}\textbf{CAC}\\ ($\zeta=1$)\end{tabular}} & \multicolumn{1}{c|}{\begin{tabular}[c]{@{}c@{}}\textbf{CAC}\\ \end{tabular}} & \textbf{PPO} & \textbf{TD3} & \textbf{SAC} & \multicolumn{1}{c}{\begin{tabular}[c]{@{}c@{}}\textbf{CAC}\\ ($\zeta=1$)\end{tabular}} & \multicolumn{1}{c}{\begin{tabular}[c]{@{}c@{}}\textbf{CAC}\\ \end{tabular}} \\
\midrule
Ant         & 1979     & 4979       & 7793       & 7160        & \textbf{1811}  & 359    & 510        & 642        & 591         & \textbf{297}          \\
HalfCheetah & $\times$ & 2337       & 3717       & 4200        & \textbf{1870}  & $\times$ & 331        & 425        & 397         & \textbf{286}          \\
Hopper      & \textbf{1598} & 3515  & 2598       & 2944        & 1944           & 318    & 609        & 454        & 394         & \textbf{279}          \\
Humanoid    & $\times$ & $\times$   & 4115       & 3092        & \textbf{2199}  & $\times$ & $\times$     & 645        & 436         & \textbf{313}          \\
Walker2d    & 1673     & 3729       & 4577       & 4310        & \textbf{1345}  & 330    & 461        & 499        & 334         & \textbf{183}          \\
\bottomrule
\end{tabular}
    }
\end{sc}
\end{center}
\end{table*}

As CAC combines concepts from ADP literature such as KL regularization and conservative learning that have not seen applications in AC, it is interesting to examine the combination against existing AC methods in challenging tasks. 
We choose a set of high dimensional continuous control tasks from the OpenAI gym benchmark suite \citep{brockman2016openai}.

For comparison, we compare CAC with twin delayed deep deterministic policy gradient (TD3) \citep{pmlr-v80-fujimoto18a} that comprehensively surveys the factors causing poor performance of actor-critic methods, to examine the cautious learning mechanism. As CAC is based on the CPI that has also inspired TRPO and PPO, we compare it with PPO which is improved over TRPO \citep{schulman1707proximal}. As PPO does not involve computing $\zeta$, we include the curves when $\zeta=1$.
We also compare with SAC \citep{pmlr-v80-haarnoja18b} which has similar architecture. 

%The objective of the experiments is to understand how CAC performs in high-dimensional continuous control tasks compared with previous works.
%We compare CAC with some off-policy and on-policy algorithms: SAC, twin delayed deep deterministic policy gradient algorithm (TD3)~\citep{pmlr-v80-fujimoto18a}, proximal policy optimization (PPO)~\citep{schulman1707proximal}, and a variant of CAC that always updates greedily, namely CAC with $\zeta=1$.
%We compare those algorithms on a range of challenging continuous control tasks from the OpenAI gym benchmark suite~\citep{brockman2016openai}.
% The benchmarks are known to be challenging for off-policy algorithms as they are hard to fine-tune and prone to oscillate. 

To better illustrate and quantify the \emph{stability} during learning, we follow \citep{zhu2020ensuring} to define the measure of performance oscillation: 
\begin{equation}
\begin{array}{l}
\forall k, \text { such that } R_{k+1}-R_{k}<0 \\
\|\mathcal{O} J\|_{\infty}=\max _{k}\left|R_{k+1}-R_{k}\right|, \\
\|\mathcal{O} J\|_{2}=\sqrt{\frac{1}{N}\sum_{k=1}^{N}\left(R_{k+1}-R_{k}\right)^{2}},
\end{array}
\end{equation}
where $N$ is the steps of the learning and $R_k$ refers to the cumulative reward reported at $k$-th evaluation.
Intuitively, $\|\mathcal{O} J\|_{\infty}$ and $\|\mathcal{O} J\|_{2}$ measure the maximum and average degradation during learning, respectively.

\subsection{Comparative Evaluation}\label{sec:comparative}

We run all algorithms with the same set of hyperparameters listed in Section \ref{apdx:hypers} of the Appendix.
All figures are plotted with statistics from 10 different random seeds, with each performing 10 evaluation rollouts every 5000 environment steps.
%Since our objective includes the evaluation of policy oscillation, we sample actions from the policy distribution in the evaluation rollouts rather than choosing the maximum a posteriori actions \citep{pmlr-v80-haarnoja18b}.

\Figref{fig:mujoco-return} shows the learning curves of the algorithms. CAC achieved comparable performance with other AC algorithms while significantly stabilized learning curves.
PPO's learning speed was the slowest among all algorithms on all environments, due to the on-policy nature of PPO which is sample-inefficient. 
Other methods were able to leverage off-policy samples to quickly learn meaningful behaviors. However, the fast learning came at a cost: except the relatively simple \texttt{HalfCheetah-v2}, on all environments these off-policy algorithms oscillated wildly, especially on the challenging \texttt{Humanoid-v2} where both PPO and TD3 failed to learn any meaningful behaviors, and the performance of SAC, CAC with $\zeta=1$ degraded frequently.
On the other hand, CAC traded off a little bit slower learning for stability, exhibiting smooth curves.
Indeed, the convergence rate of CAC is $O\!\adaBracket{e^{-(1-\gamma)\sum_{j=1}^{k}\zeta_{j}\!}}$ \citep{Vieillard-2020DCPI}, which emphasizes stability more as $\zeta \rightarrow 0$.

The comparison on stability of the algorithms can be seen from the Table \ref{tb:oscillation} that summarized the values of $\|\mathcal{O} J\|_{\infty}$ and $\|\mathcal{O} J\|_{2}$ for all algorithms.
It provided empirical support as CAC showed least oscillation during learning. This is in contrast to other off-policy algorithms oscillated wildly during learning.
%Especially in Ant-v2 task, although other off-policy algorithms exhibit huge policy oscillation, CAC achieves significantly stabler improvement.
Since CAC with $\zeta=1$ still showed huge oscillation, it can be concluded that the mixture coefficient introduced in CAC is effective in preventing significant policy degradation.

\subsection{Ablation study on mixture coefficient}

%One of the main difficulty of combining $\zeta$-based algorithms \citep{kakade-cpi,pmlr-v28-pirotta13,abbasi-improvement16} with deep networks is the algorithm tends to produce vanishly small $\zeta$ due to the long horizon, which significantly hinders learning. Hence heuristic $\zeta$ design has been introduced in \citep{Vieillard-2020DCPI}. Following this concept, we conduct an ablation study and discuss the design of $\zeta$ that better suits the entropy-regularized family of RL algorithms. 

\begin{figure}[t]
    \centering
    \includegraphics[width=0.575\columnwidth]{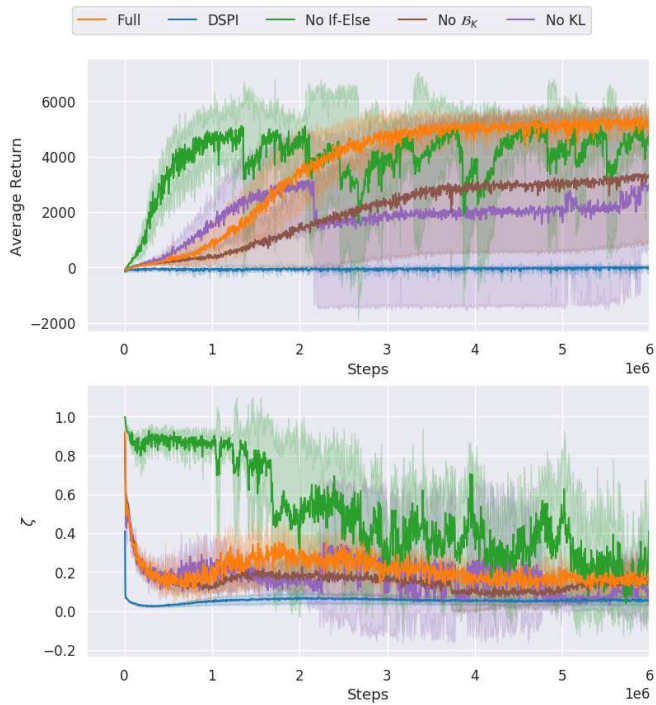}
    \caption{\textbf{Top:} The learning curves of derived methods of CAC on Ant-v2. \textbf{Bottom:} The value of the mixture coefficient $\zeta$.}
    \label{fig:ant}
\end{figure}

%Here, we study what component of the mixture coefficient design proposed in \Secref{sec:cac-mixratio} contributes the performance.
%The proposed $\zeta$ is designed to satisfy the following three principles: (1) $\zeta$ should not be too conservative, (2) $\zeta$ should be 0 whenever a significant degradation is detected, (3) $\zeta$ should be approximated with off-policy samples but with controllable loss.
In this section we conduct an ablation test to study the effectiveness of CAC as well as the proposed $\zeta$ design in Section \ref{sec:zeta_design}.
 We compare the following setup:
\begin{enumerate}
    \item \textbf{Full. } This is CAC with the proposed $\zeta^{\text{CAC}}$ in \EqNumber{eq:ERCAC-zeta2}. The curve is same as in \Figref{fig:mujoco-return}.
    \item \textbf{No KL. } We remove KL regularization from \EqNumber{eq:cac}. This corresponds to  SAC with the cautious actor.
    \item \textbf{DSPI. } This corresponds to \emph{Deep Safe Policy Iteration} \citep{zhu2020ensuring} that uses the $\zeta$ suggested by \citep{Vieillard-2020DCPI}.
    \item \textbf{No } $\onrbuf$. We replace the on-policy replay buffer $\onrbuf$ with off-policy $\rbuf$ as suggested by \citep{Vieillard-2020DCPI}. 
    \item \textbf{No If-Else. } This corresponds to removing the \texttt{if} term in \EqNumber{eq:ERCAC-zeta2} and learns with only the \texttt{else} condition.
\end{enumerate}

As is obvious from \Figref{fig:ant}, while removing the \texttt{if-else} judgement accelerated learning, it ignored the warning from $M < 0$ that the updated policy was poorer. The consequent curve oscillated drastically as the result of aggressive $\zeta$ in the bottom image. 

Removing KL regularization induced learning curve similar with \textbf{Full} in the beginning, but the performance degraded significantly since the middle stage and failed to recover. This is probably due to the policies were corrupted by error. 

Using off-policy replay buffer $\rbuf$ demonstrated stable learning. This is expected as the agent was forced to learn cautiously by the CAC mechanism. On the other hand, the learning was slow as off-policy samples were not as informative as on-policy ones, as we observed small $\zeta$ values in the bottom figure.

It is most interesting to examine the DSPI case where $\zeta$ was set according to the suggestion of \citep{Vieillard-2020DCPI}. Though this scheme works well in Atari games, the resulting algorithm failed to learn any meaningful behaviors in the challenging control tasks with continuous action spaces. This is because they estimate with the entire off-policy replay buffer $\rbuf$ which tends to produce very large estimate of $\tilde{\Delta}^{\pi_{k+1}}_{\pi_{k}}$, leading to vanishingly small $\zeta^{\text{DSPI}}$ and subsequent poor performance.

\iffalse
{\it no KL} does not introduce KL regularization, i.e. the KL coefficient $\tau$ is set to 0.
Although {\it no KL} shows a slightly faster learning around the beginning compared to {\it full}, it fails to prevent policy degradation.

{\it Deep safe policy iteration (DSPI)} uses the value of $\zeta$ proposed by \citet{Vieillard-2020DCPI}.

Compared to {\it full}, {\it DSPI} shows too small $\zeta$ and fails to solve the task as shown in \Figref{fig:ant}. 
This indicates that the $\zeta$ used in {\it full} satisfies the first principle.

{\it no detection} is a variant of CAC which removes the {\it if} statement from \Eqref{eq:CAC-zeta2}. 
While it makes the $\zeta$ large and shows faster learning, the learning curve significantly oscillates.
This shows that the trick induced to satisfy the second principle is essential to prevent policy oscillation.

{\it no $\onrbuf$} uses the normal replay buffer $\rbuf$ instead of $\onrbuf$ in \Eqref{eq:CAC-zeta2}.
It does mitigate the oscillation, but the learning curve is slower than {\it full}.
Since $\rbuf$ includes older samples than $\onrbuf$ and those may be irrelevant to the current stationary distribution, replacing $\onrbuf$ by $\rbuf$ can make the algorithm too conservative.
This justifies the third principle.
\fi

\section{Conclusion}\label{sec:conclusion}

We have presented CAC, a novel actor-critic algorithm by introducing several concepts from the approximate dynamic programming literature.
%Those concepts have seen significant advances recently but has not been carried over to the actor-critic setting. 
The cautiousness in CAC consists in the doubly conservativeness: the actor follows conservative policy iteration \citep{kakade-cpi} that ensures monotonic improvement and the critic exploits conservative value iteration \citep{kozunoCVI,vieillard2020leverage} that has been shown to yield state-of-the-art guarantees in ADP literature. 

Our key observation was by introducing an entropy-regularized critic the unwieldy interpolated actor update can be simplified significantly while still ensuring robust policy improvement.
CAC performed comparable to the state-of-the-art AC methods while significantly stabilized learning on the benchmark control problems with high dimensional continuous state-action spaces.

%To tackle the difficulty of the the classic conservative policy iteration to be compatible with deep networks and long horizon learning, we introduced the novel $\zeta$ design. 
%The ablation study verified the effectiveness of the design. As a contrast, the design from a very recent work \citep{Vieillard-2020DCPI} failed to learn any meaningful behaviors.

An interesting future direction is to incorporate other entropy for different purposes. For example, $\alpha$-divergence could be used to achieve sparse optimal policies.

\acks{
This work is partly supported by JSPS KAKENHI Grant Number 21J15633 and 21H03522. 

}

\bibliography{refs}

\clearpage

\appendix

\section{Appendix}\label{sec:appendix}

In this appendix we provide missing proofs and implementation details.
Specifically, we present Theorem 1 for CAC convergenc in Section \ref{apdx:ER-conservative-policy-iteration}, Proposition 4 for calculating CAC actor gradient expression in Section \ref{apdx:cac-gradient} and implementation details in Section \ref{apdx:impl}.

\subsection{Proof for the CAC convergence}\label{apdx:ER-conservative-policy-iteration}

% While the proofs are extensions of the SAC \citep{pmlr-v80-haarnoja18b}, we present them here for completeness.
We prove the convergence of CAC using policy iteration style argument.
Similar proofs have also been used in \cite[Theorem 4]{pmlr-v80-haarnoja18b}.
The following lemmas establish the convergence of the policy evaluation and policy improvement of CAC, respectively.

\begin{lemma}[CAC Policy Evaluation]
Given the current policy $\pi$ and a baseline policy $\basepol$, the policy evaluation step of CAC is formulated as:
\begin{align}\label{eq:cac-eval}
Q_{\pi_{k+1}} \leftarrow \rewardNoT + \gamma \adaBracket{\mathbf{P} \avg{\pi}[Q_{\pi_{k}}(s, a)] + \entreg{\basepol}{\pi}(s)}.
\end{align}
Consider an initial Q value $Q_0: \state \times \action \rightarrow \reals$ with $|\action|<\infty$. 
With the repeated application of \EqNumber{eq:cac-eval}, the sequence $Q_k$ converges to the following entropy-regularized Q-value $\cviaval{\basepol}{\pol}$ as $k\rightarrow \infty$.
\begin{align}
\cviaval{\basepol}{\pol} (\sv, \av) :=  \avg{d^\pi}\left[ \sum_{t=0}^{\infty} \gamma^{t} \big(r_t + \mathcal{I}_{\basepol}^{\pi}\left(\stp\right) \big) \big | s_{0}=\sv, a_{0}=\av \right].
\end{align}
\label{lemma:cac-eval}
\end{lemma}

\begin{proof}
Define the entropy augmented reward as ${\reward^{\pi}_{\basepol}}(\sv, \av) \triangleq \rewardNoT  + {\entreg{\basepol}{\pi}}(s)$ and rewrite the update rule as:
\begin{align}
Q_{\pi_{k+1}} \leftarrow \reward^{\pi}_{\basepol}(\sv, \av) + \gamma\mathbf{P} \avg{\pi}[Q_{\pi_{k}}(s, a)].
\end{align}
With the assumption $|\action|<\infty$ for bounded reward, we can apply the standard convergence results for policy evaluation~\citep{Sutton-RL2018}.
\end{proof}

\begin{lemma}[CAC Policy Improvement]
Given the current policy $\pi$, a baseline policy $\basepol$ and the updated policy $\pi_{\text{new}}$. CAC has the following policy update:
\begin{equation}
\begin{aligned}
&\pi_{\text{new}} = (1 - \zeta)\pi + \zeta \hat{\pi}, \\
&\text{where } \,\,  \hat{\pi}(\av \mid \sv) = \frac{\basepol^{\alpha}(\av \mid \sv) \exp \left(\beta \cviaval{\basepol}{\pol} (\sv, \av)\right)}{Z(\sv)},
\end{aligned}
\end{equation}
with $\zeta \in [0, 1]$.
Then, $\cviaval{\basepol}{\pol_{\text{new}}}(\sv, \av) \geq \cviaval{\basepol}{\pol}(\sv, \av)$ for all $(\sv, \av) \in \state\times\action$ with $|\action|<\infty$.
\label{lemma:cac-improvement}
\end{lemma}

\begin{proof}
Consider a function $f: \zeta \to \reals$ with $\zeta \in [0, 1]$:
\begin{equation}
f(\zeta) = \avg{\av \sim \pi_{\text{new}}}\left[\cviaval{\basepol}{\pi}(\sv, \av) \right] + \entreg{\basepol}{\pi_\text{new}}\left(\sv\right).
\end{equation}
From the definition of $\hat{\pi} \!=\! \underset{\pi}{\argmax} \;\avg{\av \sim \pi}\left[\cviaval{\basepol}{\pi}(\sv, \av) \right] \!+\! \entreg{\basepol}{\pi}\left(\sv\right)$, $f(\zeta)$ takes the maximum value when $\zeta=1$.
The first and the second derivative of $f(\zeta)$ w.r.t. $\zeta$ are:
\begin{equation}
    \begin{aligned}
    f^{'}(\zeta) = \sum_{a}&\left(\hat{\pi}(a|s) - \pi (a|s) \right) \left(\cviaval{\basepol}{\pi}(s, a)+\tau \log\basepol (a|s) \right) \\
    &\left. - (\sigma+\tau)\log((1 - \zeta)\pi(a|s) + \zeta \hat{\pi} (a|s)\right),\\
    \end{aligned}
\end{equation}

\begin{equation}
    \begin{aligned}
    f^{''}(\zeta) = -(\sigma + \tau)\sum_{a}\frac{\left(\hat{\pi}(a|s) - \pi(a|s) \right)^2}{(1 - \zeta)\pi(a|s) + \zeta \hat{\pol}(a|s)} \leq 0.
    \end{aligned}
\end{equation}

Thus, the function $f$ is concave in $\zeta \in [0, 1]$. 
Since $f(\zeta)$ takes the maximum value with $\zeta=1$, $f(\zeta)$ is monotonically increasing in $\zeta$ and $f(0)\leq f(\zeta)$.

Therefore, the following inequality about the entropy-regularized V-value $\cvival{\basepol}{\pol}(\sv)$ holds:
\begin{equation}\label{eq:er-q-ineq}
\begin{aligned}
\cvival{\basepol}{\pol}(\sv)=&\avg{\av \sim \pi}\left[\cviaval{\basepol}{\pi}(\sv, \av) \right] + \entreg{\basepol}{\pi}\left(\sv\right) \\
\leq
&\avg{\av \sim \pi_{\text{new}}}\left[\cviaval{\basepol}{\pi}(\sv, \av) \right] + \entreg{\basepol}{\pi_\text{new}}\left(\sv\right). \\
\end{aligned}
\end{equation}

By repeatedly applying \Eqref{eq:er-q-ineq}, we obtain the following inequalities:
\begin{equation}
\begin{aligned}
\cviaval{\basepol}{\pol}(\sv, \av)
&= \reward(\sv, \av) + \gamma\mathbf{P}\left[\cviaval{\basepol}{\pol}(s, a)+\entreg{\basepol}{\pol}(s)\right]\\
&\leq \reward(\sv, \av) + \gamma\mathbf{P}\left[\avg{\pi_{\text{new}}}[\cviaval{\basepol}{\pol}(s, a)]+\entreg{\basepol}{\pol_\text{new}}(s)\right]\\
& \; \vdots\\
& \leq \cviaval{\basepol}{\pol_\mathrm{new}}(\sv, \av).
\end{aligned}
\end{equation}

Convergence to $\cviaval{\basepol}{\pol_\mathrm{new}}$ follows from \Lemmaref{lemma:cac-eval}.
\end{proof}

Combining the policy evaluation and policy improvement, we are now ready to prove Theorem 1.

\textbf{Theorem 1}
\emph{Repeated application of CAC Eq. (\ref{eq:cac}) on any initial policy $\pi$ will make it converges to the entropy regularized optimal policy $\pol^{*}(a|s) = \frac{\exp\adaBracket{\frac{1}{\kappa} Q^{*}(s,a)}}{\int_{a\in\mathcal{A}} \exp\adaBracket{\frac{1}{\kappa} Q^{*}(s, a) } }$.}

\begin{proof}
According to \Lemmaref{lemma:cac-eval} and 
\Lemmaref{lemma:cac-improvement}, the entropy-regularized Q-value at $k$-th update satisfies $\cviaval{\pol_{k-1}}{\pol_k}(s,a) \geq \cviaval{\pol_{k-1}}{\pol_{k-1}}(s,a)$.
Given bounded reward, $\cviaval{\pol_k}{\pol_k}$ is also bounded from above and the sequence converges to a unique $\pol^{*}$. 
Note that when reaching the optimum the KL regularization term becomes $0$.
Hence, using the same iterative argument as in the proof of \Lemmaref{lemma:cac-improvement}, we get $\cviaval{\pol^{*}}{\pol^{*}}(\sv, \av) > \cviaval{\pol}{\pol}(\sv, \av)$ for all $(\sv, \av)\in \state\times\action$ and any $\pi$.
By \cite{ziebart2010-phd,pmlr-v80-haarnoja18b}, the optimal policy is entropy-regularized and hence has the softmax form $\pol^{*}(a|s) = \frac{\exp\adaBracket{\frac{1}{\kappa} Q^{*}(s,a)}}{\int_{a\in\mathcal{A}} \exp\adaBracket{\frac{1}{\kappa} Q^{*}(s, a) } }$ (recall from Eq. (\ref{eq:cac}) that $\kappa$ is the weight coefficient of entropy).
The convergence of general interpolated policy to the optimal policy follows the argument of \cite{Scherrer2014-localPolicySearch}.
\end{proof}

\subsection{Proof for the CAC gradient}\label{apdx:cac-gradient}

In this subsection we derive the gradient expression for CAC. 
For the ease of reading we rephrase the proposition here:

\textbf{Proposition 4}
    \emph{Let the actor network be parametrized by weights $\phi$ and critic by $\theta$.
    Define $\mathcal{G}Q_{\bar{\phi},\theta}$ as the greedy policy with respect to the CAC critic.
    The subscript $\bar{\phi}$ comes from the baseline policy introduced by KL divergence.
    Then the gradient of the actor update can be expressed as:}
    \begin{align}
        \begin{split}
        &\nabla_{\pparams} \avg{\substack{\sv\sim\rbuf\\\av\sim\pol_{\pparams}}}
        \left[ D^{\phi}_{\bar{\phi}}
        -\frac{\beta}{1+\const}Q_\theta(\sv, \av)
        \right], \\
        &\text{where } D^{\phi}_{\bar{\phi}} =  \log \pi_{\pparams}\left(\av  \mid \sv \right)
        -\frac{\alpha+\const}{1+\const}\log\pi_{\bar{\pparams}}(\av; \sv ) \\
        &\const = \frac{1-\zeta}{\zeta}\cdot
        \frac{\pi_{\bar{\phi}}(\av\mid \sv)}{\gr Q_{\bar{\phi}, \theta}(\av \mid \sv)} .
        \end{split}
    \end{align}
\begin{proof}
Using the reparameterization trick $a \!=\! f_{\phi}(\nv; \st)$ with $\nv$ a noise vector \citep{pmlr-v80-haarnoja18b}, the gradient of \Eqref{eq:phi} can be expressed as:
\begin{equation}\label{eq:cac_policy_loss_prev}
\begin{aligned}
\hat{\nabla}_{\phi} &J_{\pi}(\phi) \!=\!
\nabla_{\phi} \log \pi_{\phi}\left(\av  \mid \sv \right) 
+ \nabla_{\av} \log \pi_{\phi}\left(\av  \mid \sv \right)
\nabla_{\phi} f_{\phi}\left(\nv ; \st \right)\\
&-\nabla_{\av}\log\left(\!
(1 - \zeta)\pi_{\bar{\phi}}(\av \mid \sv)\right. 
 \!+\! \left.\zeta \gr Q_{\bar{\phi}, \theta}\left(\av \mid \sv\right) \!
\right) \!
\nabla_{\phi} f_{\phi}\left(\nv ; \st \right).
\end{aligned}
\end{equation}
We expand the term $\nabla_{\av}\!\log{\left(\! (1 - \zeta)\pi_{\bar{\phi}}(\av \mid \sv) + \zeta \gr Q_{\bar{\phi}, \theta}\left(\av \mid \sv\right)\!\right)}$ by using that $\nabla_{x_{i}}\!\log \left(\sum_{i} \exp x_{i}\right)$  $=\frac{\exp x_{i}}{\sum_{j} \exp x_{j}}$.
%The partial derivative of $\log \left(\sum_{i} \exp x_{i}\right)$ w.r.t. $x_i$ is 
%\begin{equation}\label{eq:logsumexp}
%\frac{\partial}{\partial x_{i}} \log \left(\sum_{i} \exp x_{i}\right)=\frac{\exp x_{i}}{\sum_{j} \exp %x_{j}}
%\end{equation}

Let $\exp{\left(C_1(\av)\right)}=(1 - \zeta)\pi_{\bar{\phi}}(\av \mid \sv)$ and $\exp{\left(C_2(\av)\right)} = \zeta \gr_{\pi_{\bar{\phi}}} Q_{\theta}\left(\av \mid \sv\right)$.
We have the following transformation:
\begin{equation}\label{eq:derive}
    \begin{aligned}
    &\nabla_{\av}\log{\left(
    (1 - \zeta)\pi_{\bar{\phi}}(\av \mid \sv) + \zeta \gr Q_{\bar{\phi}, \theta}\left(\av \mid \sv\right)
    \right)} \\
    =& \nabla_{\av}\log{\left(\exp{\left(C_1(\av)\right)} + \exp{\left(C_2(\av)\right)} \right)}\\
    =& \frac{
    \left(D\frac{\partial}{\partial \av}\pi_{\bar{\phi}}(\av \mid \sv)
    +\alpha\frac{\partial}{\partial \av}\pi_{\bar{\phi}}(\av \mid \sv) 
    +\beta\frac{\partial}{\partial \av} Q_\theta(\sv, \av)\right)
    }{
    1 + D
    }, \\
    &\text{where } D = \exp{\left(C_1(\av)- C_2(\av)\right)}.
    \end{aligned}
\end{equation}
After replacing $D$ with $\const$ and inserting \Eqref{eq:derive} into \Eqref{eq:cac_policy_loss_prev}, we obtain \Eqref{eq:cac_sgd_grpi}.

\end{proof}

\subsection{Implementation details}\label{apdx:impl}

This section presents implementation details of CAC with deep networks.
Pseudo-code is provided in Algorithm \ref{alg:cac}.

\paragraph{On-policy replay buffer} 
To make the algorithm off-policy, we approximate the on-policy samples with on-policy replay buffer $\onrbuf$ which stores $K$ recent samples where $K$ is smaller than the size of the main replay buffer $\rbuf$.

\paragraph{Advantage estimation}
While it is possible to simply use the entropy-regularized advantage function $A_{\mathcal{I}}(s,a) \!=\! Q_{\mathcal{I}}(s,a) \!-\! V_{\mathcal{I}}(s)$ for computing $\zeta$, we are interested in studying the guidance of $\zeta$ when no entropy is involved since it might provide a more informative gradient improving direction. This corresponds to the case of \citep{kakade-cpi,pmlr-v28-pirotta13}.
To this end, we train another Q-network $Q_\omega$ by solving: 
\begin{equation}\label{eq:cac_A_grad}
\begin{aligned}
\omega \leftarrow &\arg \min \avg{\rbuf}\left[\left({Q}_{\omega}\left(\sv , \av \right)-y \right)^{2}\right],\\
\text{where }\, \, y &= r + \gamma \left(\avg{\av \sim \pi_\phi(\cdot|\svp)}\left[Q_{\bar{\omega}} (\svp, \av)\right] \right),\\
\end{aligned}
\end{equation}

where $\bar{\omega}$ is the target network.
Then we approximate the advantage as $A^{\pi}(s,a) = Q_{\omega}(s,a) - \avg{\pi_{\phi}}[Q_{\omega}(s, a)]$.
While the advantage estimation is expected to be further improved with the recent generalized advantage estimation, we found that the above simple implementation is sufficient to stabilize the learning.

\paragraph{Target smoothing}

For the target Q-networks $Q_{\bar{\theta}}$ and $Q_{\bar{\omega}}$, we update the parameters using the moving average \citep{pmlr-v80-haarnoja18b}:
\begin{equation}
\begin{aligned}
\bar{\theta} \leftarrow \ltheta \, \theta + (1-\ltheta) \bar{\theta},\\
\bar{\omega} \leftarrow \lomega \, \omega + (1-\lomega) \bar{\omega},\\
\end{aligned}
\end{equation}

where $\ltheta$ and $\lomega$ are the target smoothing coefficients.
In the mixing step, we use the previous policy $\pi_{\bar{\phi}}$ rather than the current policy $\pi_{\phi}$ to stabilize the training:
\begin{align*}
\avg{\rbuf}\left[{D}_{{KL}}\left(\pi_{\phi} \| (1-\zeta)\pol_{\bar{\pparams}} + \zeta \hat{\pol} \right)\right].
\end{align*}

Thus, the target policy $\pol_{\bar{\pparams}}$ corresponds to the monotonically improved policy in the CPI algorithm that is not updated when performance oscillation happens. 
To reflect this fact, we update the weight of the target policy network as:
\begin{equation}\label{eq:phi_bar}
\bar{\phi} \leftarrow \zeta \lphi \, \theta + (1-\zeta\lphi) \bar{\phi},
\end{equation}
where $\lphi$ is the target smoothing coefficient.

\paragraph{Normalization factor estimation}

Since CAC algorithm requires the density of the reference policy $\hat{\pi}(\av \mid \sv) = {\pi_{\bar{\phi}}^{\alpha}(\av \mid \sv) \exp \left(\beta Q_{\theta} (\sv, \av)\right)}(Z(\sv))^{-1}$, we need to estimate the normalization factor $Z(s)$.

A simple approach to estimate $\partition$ is by Monte-Carlo sampling with some distribution $q$ that is easier to sample from:
\begin{equation}\label{eq: Z}
    {Z}(\sv) = \avg{q} \left[ \frac{{\pi_{\bar{\phi}}}(\av \mid \sv)^{\alpha} \exp \left(\beta Q_{\theta}(\sv, \av)\right)}{q(\av \mid \sv)} \right].
\end{equation}
The closer $q(\cdot \mid \sv)$ and the reference policy $\hat{\pi}(\cdot \mid \sv)$ are, the better the accuracy of the $Z(\sv)$ approximation.

Theorem \ref{theorem:ck} indicates that by choosing the current policy $\pi$ as the proposal distribution, we can control the closeness of the two distributions and the accuracy of the MC approximation via changing the entropy regularization weighting coefficients.
We empirically study the effectiveness of entropy regularization against the closeness and the accuracy when $\zeta\leq 1$

We use the pendulum environment from \citet{pmlr-v97-fu19a} where the dynamics are discretized so that we can compute the oracle values such as the KL divergence between the current and the reference policy.
The hyperparameters used in the experiment is listed in Section \ref{apdx:hypers}.
\Figref{fig:pend} shows how the learning behavior of CAC changes when the interpolation coefficient $\zeta$ and KL regularization weight $\tau$ vary:
When KL regularization is present,  the approximation quality of $Z(s)$ is improved significantly.
%This justifies our conservative critic implementation for the conservative actor.

\begin{figure*}[t]
\vskip 0.2in
\begin{center}
\centerline{\includegraphics[width=1.0\columnwidth]{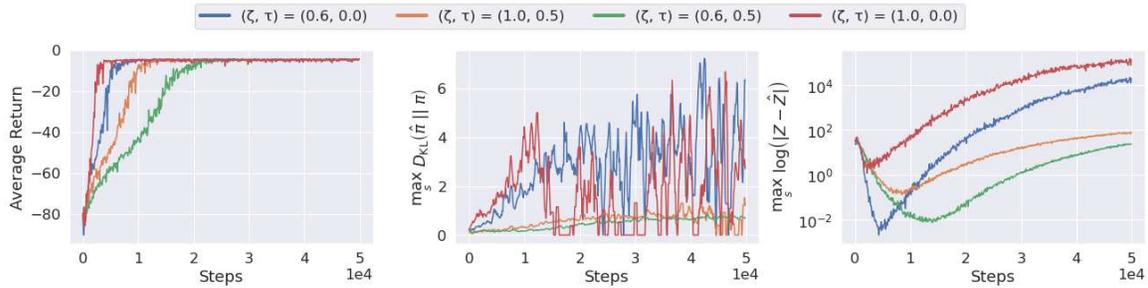}}
\caption{\textbf{Left:} The learning curves of CAC with different parameters on pendulum. \textbf{Middle:} The maximum KL divergence over all the states. \textbf{Right:} The maximum error of log-scaled $Z(s)$ approximation over all the states.}
\label{fig:pend}
\end{center}
\vskip -0.2in
\end{figure*}

\begin{algorithm}[tb]
\begin{algorithmic}[1]
\caption{Cautious Actor-Critic} \label{alg:cac}
% \REQUIRE 
\STATE Initialize parameter vectors $\theta$, $\phi$, $\bar{\theta}$, $\bar{\phi}$, $\omega$, $\bar{\omega}$
\STATE Initialize variable $\approxadv$ and $\approxadv^{\mathrm{MaxDiff}}$
\FOR{each iteration}
    \STATE Collect transitions by $\pi_\theta$ and add them to $\rbuf$ and $\onrbuf$
	\FOR{each gradient step}
	    \STATE Update $\theta$ with one step of SGD using \EqNumber{eq:theta}
	    \STATE Update $\omega$ with one step of SGD using \EqNumber{eq:cac_A_grad}
    	\STATE Update $\approxadv$ and $\approxadv^{\mathrm{MaxDiff}}$ using \EqNumber{eq:ERCAC-zeta2}
	    \STATE Update $\phi$ with one step of SGD using \EqNumber{eq:phi}
    	\STATE Update $\zeta$ using \EqNumber{eq:ERCAC-zeta}
    	\STATE $\bar{\vparams}\leftarrow \ltheta \vparams + (1-\ltheta)\bar{\vparams}$
    	\STATE $\bar{\omega}\leftarrow \lomega \omega + (1-\ltheta)\bar{\omega}$
    	\STATE $\bar{\phi} \leftarrow \zeta\lphi \phi+(1-\zeta\lphi) \bar{\phi}$
	\ENDFOR
\ENDFOR
\end{algorithmic}
\end{algorithm}

\begin{table}[h]
\renewcommand{\arraystretch}{1.1}
\centering
\caption{Hyperparameters of off-policy algorithms in mujoco tasks}
\label{tab:shared_params_1}
\vspace{1mm}
\begin{tabular}{l l| l }
\toprule
\multicolumn{2}{l|}{Parameter} &  Value\\
\midrule
\multicolumn{2}{l|}{\it{Shared}}& \\
& optimizer &Adam \\
& learning rate & $10^{-3}$\\
& discount factor ($\gamma$) &  0.99\\
& replay buffer size ($\rbuf$) & $10^6$\\
& number of hidden layers  & 2\\
& number of hidden units per layer & 256\\
& number of samples per minibatch & 100\\
& activations & ReLU\\
\midrule
\multicolumn{2}{l|}{\it{TD3}}& \\
& Stddev for Gaussian noise & 0.1  \\
& Stddev for target smoothing noise & 0.2  \\
& policy delay & 2 \\
\midrule
\multicolumn{2}{l|}{\it{SAC}}& \\
& entropy coefficient ($\kappa$) & 0.2\\
& $\bar{\theta}$ smoothing coefficient & 0.995\\
\midrule
\multicolumn{2}{l|}{\it{CAC}}& \\
& entropy coefficient ($\kappa$) & 0.2\\
& KL coefficient ($\tau$) & 0.1\\
& $\bar{\theta}$ smoothing coefficient ($\ltheta$) & 0.995\\
& $\bar{\omega}$ smoothing coefficient ($\lomega$) & 0.995\\
& $\bar{\phi}$  smoothing coefficient ($\lphi$) & 0.9999\\
& $\lA$ & $0.01$\\
& $\lAslow$ & $0.001$\\
& size of $\onrbuf$ & $1000$  \\
& \texttt{if-else} update & $c=M$ \\
\bottomrule
\end{tabular}
\end{table}

\begin{table}[h]
\renewcommand{\arraystretch}{1.1}
\centering
\caption{Hyperparameters of PPO in mujoco tasks}
\label{tab:shared_params_2}
\vspace{1mm}
\begin{tabular}{l l| l }
\toprule
\multicolumn{2}{l|}{Parameter} &  Value\\
\midrule
\multicolumn{2}{l|}{\it{PPO}}& \\
& optimizer &Adam \\
& value function learning rate & $10^{-3}$\\
& policy learning rate & $3\times 10^{-4}$\\
& discount factor ($\gamma$) &  0.99\\
& number of hidden layers  & 2\\
& number of hidden units per layer & 256\\
& number of samples per minibatch & 100\\
& activations & ReLU\\
& Number of samples per update & 80 \\
& Policy objective clipping coefficient & 0.2 \\
\bottomrule
\end{tabular}
\end{table}

\begin{table}[h]
\renewcommand{\arraystretch}{1.1}
\centering
\caption{Hyperparameters of CAC in pendulum task}
\label{tab:shared_params_3}
\vspace{1mm}
\begin{tabular}{l l| l }
\toprule
\multicolumn{2}{l|}{Parameter} &  Value\\
\midrule
\multicolumn{2}{l|}{\it{CAC}}& \\
& optimizer &Adam \\
& learning rate & $10^{-3}$\\
& discount factor ($\gamma$) &  0.99\\
& replay buffer size ($\rbuf$) & $10^6$\\
& number of hidden layers  & 2\\
& number of hidden units per layer & 256\\
& number of samples per minibatch & 32\\
& activations & ReLU\\
& entropy coefficient ($\kappa$) & 0.2\\
& $\bar{\theta}$ smoothing coefficient & 0.995\\
& $\bar{\phi}$  smoothing coefficient & 0.995\\
& size of $\onrbuf$ & $1000$  \\
& \texttt{if-else} update & $c=M$ \\
\bottomrule
\end{tabular}
\end{table}

\subsection{Hyperparameters} \label{apdx:hypers}
This section lists the hyperparameters used in the comparative evaluation Section \ref{sec:comparative}.

\end{document}